%% file: neurips_2026.tex
\documentclass{article}

\usepackage[preprint]{neurips_2026}

\usepackage[utf8]{inputenc} 
\usepackage[T1]{fontenc}    
\usepackage{hyperref}       
\usepackage{url}            
\usepackage{booktabs}       
\usepackage{amsfonts}       
\usepackage{nicefrac}       
\usepackage{microtype}      
\usepackage{xcolor}         
\usepackage{amsmath, amsthm, amssymb}
\usepackage{booktabs}           
\usepackage{xcolor}             
\usepackage{placeins}
\theoremstyle{plain}
\newtheorem{proposition}{Proposition}
\newtheorem{corollary}{Corollary}
\theoremstyle{definition}
\newtheorem{definition}{Definition}
\newtheorem{assumption}{Assumption}
\theoremstyle{remark}
\newtheorem{remark}{Remark}
\usepackage{tikz}
\usetikzlibrary{arrows.meta,positioning}
\newcommand{\SCMfigscale}{0.33}
\usepackage{wrapfig}

 \usepackage{booktabs}
  \usepackage{multirow}
  \usepackage{graphicx} 

\usepackage{graphicx} 
\usepackage{tikz}
\usetikzlibrary{arrows.meta} 
\usepackage{amsmath} 

\title{Entangled by Design: Spurious Intra-Variable Signal Routing in Tabular In-Context Learners}

\author{%
  Athanasios Vlontzos \\
  Hologen AI \\
  London, UK \\
  \And
  Giorgos Papanastasiou \\
  Mathematics Research Centre \\Academy of Athens  \\
  Athens, GR \\
   \AND
  Bernhard Kainz \\
  FAU Erlangen Nuremberg \\
  Imperial College London\\ 
  Erlangen, DE \& London, UK \\
  \And
  Sotirios Tsaftaris \\
  University of Edinburgh \\
  Edinburgh, UK \\
}

\begin{document}

\maketitle
\input{abstract}
\input{section1_introduction}
\input{section2_related_work}
\input{section3_setup}
\input{section4_theory}
\input{section5_experiments}
\input{section6_robustness}
\input{section7_mitigations}
\input{section9_conclusion}
\bibliographystyle{abbrvnat}

\bibliography{refs}
\newpage
\input{appendix}
\newpage

\newpage
\input{checklist.tex}

\end{document}

%% file: abstract.tex

\begin{abstract}
Consider a model trained at a single hospital to predict patient recovery,
where the measured feature $X$ bundles the patient's true health signal ($C$)
with a systematic artefact from that hospital's equipment ($S$).
Within that hospital, the artefact correlates with outcomes through unmeasured
confounders such as patient demographics; an in-context learner rationally
routes predictions through $S$, not $C$, and fails silently when deployed at
a new hospital with different equipment.
We formalise this as \emph{spurious routing in composite representations}:
when a feature $X = [C;\,\alpha S;\,\eta]$ encodes a causal signal $C$ and
a spurious signal $S$ in distinct subspaces, the ICL cannot determine which
drives predictions.
We prove that under ridge ICL, a linear in-context learner, this routing
is unavoidable regardless of context size; TabPFN, a state-of-the-art
pretrained tabular ICL model, shows qualitatively consistent behaviour
empirically.
We derive a closed-form characterisation, $\mathrm{CSR} \propto \rho_S/\rho_C$,
confirmed at $r = 0.997$ for linear ICL and $r = 0.979$ for TabPFN.
Contrary to intuition, larger context sharpens commitment to the dominant
in-context signal, amplifying spurious routing by up to $1.74\times$; in the
high-spurious corner, more expressive models show greater vulnerability empirically
($+2.22$ CSR gap at high entanglement).
We introduce two lightweight mitigations: environment-stratified context construction
and S-swap augmentation, that require only weak environment labels and no knowledge
of the causal partition.
S-swap reduces spurious routing by $74\%$ for linear ICL and $98.8\%$ for TabPFN,
with TabPFN's causal sensitivity increasing $8.4\times$ simultaneously: the model
does not become agnostic, it reroutes through the causal signal.
\end{abstract}

%% file: section1_introduction.tex

\section{Introduction}
\label{sec:intro}
\input{fig_scm}
In-context learning (ICL) has emerged as a powerful paradigm for tabular
prediction: given a set of labelled examples as context, a model such as
TabPFN~\citep{hollmann2023tabpfn} can fit a high-quality predictor at inference
time, without gradient updates or access to the full training distribution.
This flexibility is precisely what makes ICL attractive for scientific and
biomedical applications, where labelled data arrives in heterogeneous batches
and fast adaptation is essential.

Yet this flexibility comes with a blind spot.
An in-context learner has no mechanism to evaluate \emph{why} a feature predicts
the target, only \emph{that} it does within the provided examples.
Figure~\ref{fig:scm} shows the structural causal model underlying our
analysis: an unobserved confounder $U$ drives both a spurious component $S$
and the outcome $Y$, creating an in-context correlation
$\mathrm{Corr}(S,Y)\neq 0$ that vanishes under intervention.
The observed feature $X{=}[C;\alpha S;\eta]$ conflates the true causal
signal $C$ with this spurious component $S$, so the model routes
predictions through $S$, committing no error by any in-context criterion,
yet failing silently the moment the $S$--$Y$ correlation shifts.


Consider a concrete example.
A model is trained at a single hospital to predict patient recovery, where
the measured feature $X$ bundles the patient's true health signal ($C$,
what causally drives recovery) with a systematic artefact from that
hospital's equipment ($S$, a measurement offset that varies by site);
such measurement heterogeneity is known to induce model miscalibration
at deployment even when underlying biology is unchanged~\citep{debray2019changing}.
Within that hospital, the equipment artefact correlates with outcomes
through unmeasured confounders, i.e., patient demographics that determine
both which ward a patient is admitted to and how they respond to treatment.
An ICL trained on that hospital's data routes through $S$, not $C$; when
deployed at a new hospital with different equipment, predictions collapse---
a specific instance of the broader dataset shift problem documented in
clinical AI~\citep{subbaswamy2020deployment}.
The same structure recurs across domains: gene expression measurements
where biological response is mixed with sequencing-run artefacts that
can confound clinical outcomes when correlated with study
design~\citep{leek2010batch}, clinical biomarkers blending physiology
with instrument drift, and any feature vector that bundles a causal
signal with a site-specific confound.

\noindent\textbf{Spurious routing in composite representations.}
We formalise this failure as \emph{spurious routing in composite
representations}.
The setting we study is one where an observed feature vector $X$ encodes a
causally invariant component $C$ and a spurious component $S$ in
\emph{distinct subspaces}: $X = [C;\, \alpha S;\, \eta]$.
This is distinct from the cross-variable spurious correlation studied in
invariant learning~\citep{arjovsky2019irm, peters2016causal}, where
separate features $X_1$ and $X_2$ are independently causal or spurious.
Proposition~\ref{prop:d1} shows that when $C$ and $S$ are fully entangled
in a single scalar coordinate ($D=1$), any perturbation-based routing
diagnostic is unidentifiable: the model weight cancels in the ratio.
Separated subspaces are the minimal identifiable setting.
In practice, this structure is partially induced by harmonisation methods
(e.g.\ ComBat, MMD minimisation), which redistribute signal and artefact
into approximately distinct feature groups, which is a common preprocessing step
in genomics and multi-site clinical studies.
We analyse this separated-coordinate regime throughout.
\begin{itemize}
  \item 
    We identify and formalise spurious routing in composite representations
    as a failure mode of in-context learning distinct from cross-variable
    spurious correlation, and introduce the Causal Sensitivity Ratio (CSR)
    as a perturbation-based diagnostic that quantifies routing in the minimal
    identifiable setting of separated causal/spurious subspaces
    (Section~\ref{sec:theory}).

  \item 
    We prove that under in-context ridge regression---a linear ICL model
    in which the predictor is the ridge estimator computed over the context
    window~\citep{akyurek2022learning}---with orthogonal causal and spurious
    subspaces,
    $\mathrm{CSR}^{\mathrm{pop}} = \frac{\alpha^2}{\alpha^2+\lambda}
    \cdot \frac{\rho_S}{\rho_C} \cdot (1+\lambda)\kappa(\delta)$,
    and that spurious routing is unavoidable under in-context confounding
    for any regularisation strength
    (Proposition~\ref{prop:main},
    Corollaries~\ref{cor:mono}--\ref{cor:persist}).

  \item 
    We confirm all theoretical predictions on a controlled structural causal
    model. The Pearson correlation between observed CSR and the empirical
    predictor $\hat\rho_S/\hat\rho_C$ reaches $r{=}0.997$ for linear ICL
    and $r{=}0.979$ for TabPFN. We further show that (i)~larger context
    amplifies routing in the high-spurious regime ($+1.74\times$), and
    (ii)~TabPFN exceeds linear ICL in causal alignment across the majority
    of settings, yet is more vulnerable than linear ICL in the
    high-entanglement corner ($+2.22$ CSR gap, concentrated in 15 of 64
    grid settings) (Section~\ref{sec:experiments}).

  \item 
    We propose two lightweight interventions---environment-stratified context
    construction and S-swap augmentation---that reduce spurious routing by
    up to $98.8\%$ on TabPFN and recover out-of-distribution accuracy without
    requiring knowledge of the causal partition
    (Section~\ref{sec:mitigations}).
\end{itemize}

%% file: fig_scm.tex
\begin{wrapfigure}{r}{\SCMfigscale\linewidth}
  \vspace{-26pt}
  \centering
  \resizebox{\linewidth}{!}{
  \begin{tikzpicture}[
      >=Stealth,
      font=\sffamily,
      node distance=1.5cm,
      causal/.style={circle, draw=teal!80!black, very thick, fill=teal!5, minimum size=1.4cm, align=center},
      unobs/.style={circle, draw=gray, very thick, dashed, fill=gray!5, minimum size=1.4cm, align=center},
      spurious/.style={circle, draw=brown!80!black, very thick, fill=brown!5, minimum size=1.4cm, align=center},
      composite/.style={rectangle, rounded corners, draw=blue!70!black, very thick, fill=blue!5, minimum width=4.5cm, minimum height=1.2cm, align=center},
      outcome/.style={circle, draw=violet, very thick, fill=violet!5, minimum size=1.4cm, align=center}
  ]

  \node[composite] (X) at (0,0) {\textbf{X} = [\textbf{C} ; $\alpha$\textbf{S} ; $\eta$] \\[-0.2ex] \small observed composite};
  \node[causal] (C) at (-3.8, 2.5) {\textbf{C} \\[-0.2ex] \small causal};
  \node[spurious] (S) at (3, 2.5) {\textbf{S} \\[-0.2ex] \small spurious};
  \node[unobs] (U) at (0, 5) {\textbf{U} \\[-0.2ex] \small unobserved};
  \node[outcome] (Y) at (4.5, -0.5) {\textbf{Y} \\[-0.2ex] \small outcome};

  \draw[->, draw=gray, thick] (C) to (X.165);
  \draw[->, draw=gray, thick] (S) to node[pos=0.6, left=0.1cm, text=gray] {\footnotesize $\times \alpha$} (X.15);
  \draw[->, draw=gray, thick] (U) to (S);
  \draw[->, draw=gray, thick, dashed] (U) to[out=-10, in=80] (Y);

  \draw[->, draw=teal!80!black, very thick] (C) to[out=25, in=135] node[midway, below=0.2cm, text=teal!80!black] {causal path} (Y);

  \draw[->, draw=orange!80!red, thick, dashed] (S) to[bend left=15] node[midway, right=0.1cm, text=orange!80!red, align=left] {\small in-ctx \\[-0.5ex] \small corr.} (Y);

  \node[text=orange!80!red, font=\itshape, below=0.3cm of X] {ICL routes through S when $\operatorname{Corr}(S,Y) \neq 0$};

  \end{tikzpicture}
  }
  \caption{\small SCM: $U$ confounds $S\!\to\!Y$
    in-context; ICL sees $X{=}[C;\alpha S;\eta]$
    and routes through $S$.}
  \label{fig:scm}
  \vspace{-15pt}
\end{wrapfigure}
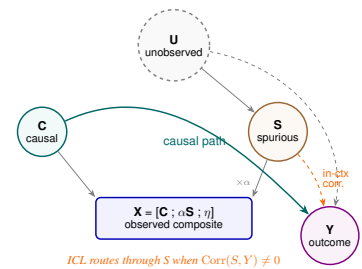

%% file: section2_related_work.tex

\section{Related Work}
\label{sec:related}

\noindent\textbf{In-context learning theory.}
The capacity of transformers to learn from context has been studied from
algorithmic, Bayesian, and statistical perspectives.
\citet{garg2022transformers} showed empirically that transformers trained
from scratch can match optimal least-squares estimators on linear regression,
establishing the simple-function-class framework that underpins our linear
ICL baseline.
\citet{xie2022icl} offered a theoretical account of this phenomenon as
implicit Bayesian inference over latent concepts, while
\citet{akyurek2022learning} and \citet{vonoswald2023transformers} proved
constructively that transformers can implement gradient descent and ridge
regression in-context.
For tabular data specifically, \citet{muller2022pfn} introduced Prior-Data
Fitted Networks (PFNs), the architecture underlying
TabPFN~\citep{hollmann2023tabpfn}, which approximate posterior predictive
distributions by meta-training on synthetic priors.
A key assumption shared across this line of work is that the in-context
distribution is \emph{informative}: examples are drawn from the task of
interest and their labels are generated by the same mechanism that governs
the query.
We show that when the context is confounded (when spurious correlations
contaminate the in-context signal via a shared unmeasured confounder), these
otherwise near-optimal learners systematically route through the wrong
signal, and that this failure is provably irreducible by standard means.

\noindent\textbf{Invariant learning and causal representations.}
The problem of learning predictors that generalise across environments has
been studied extensively in the invariant learning
literature~\citep{peters2016causal, arjovsky2019irm}.
Invariant Causal Prediction~\citep{peters2016causal} identifies causal
predictors by requiring that residuals are identically distributed across
environments; IRM~\citep{arjovsky2019irm} relaxes this to a gradient-based
penalty applicable at training time.
Both approaches assume that spurious and causal features occupy
\emph{distinct} input coordinates: the cross-variable setting.
The composite-representation setting we study, where a single observed feature
$X = C + \alpha S$ encodes causal and spurious signals in distinct subspaces,
falls outside the scope of these methods: neither environment
stratification nor gradient penalties can disentangle $C$ from $S$ when
both enter $X$ and neither is directly observed.
A related challenge arises in causal effect estimation for structured
treatments, where \citet{corcoll2024contrastive} show that contrastive
representation learning can isolate causal factors within a composite
observable, a complementary approach to our perturbation-based diagnostic.
Our work brings the invariance perspective to in-context learners,
characterising the failure theoretically and proposing context-level
mitigations that require no access to the causal partition and no
modification to the model.

\noindent\textbf{Counterfactual and spurious-attribute augmentation.}
S-swap augmentation is conceptually related to counterfactual data
augmentation~\citep{kaushik2020learning}, which constructs
training pairs that differ only in a specified attribute while holding
the label fixed, to reduce reliance on spurious cues.
The key distinction is the setting: prior work applies augmentation at
\emph{training time} to update model weights, whereas our S-swap operates at
\emph{context construction time} and requires no model update, making it
directly applicable to pretrained in-context learners such as TabPFN.
Our theoretical analysis of \emph{why} the augmentation works (driving
in-context $\mathrm{Corr}(S, Y)$ toward zero at the point level) also
provides a mechanistic account absent from prior empirical treatments.

\noindent\textbf{Domain generalisation baselines.}
Methods such as GroupDRO~\citep{sagawa2020distributionally},
VREx~\citep{krueger2021out}, and density-weighted invariance approaches
address distribution shift by reweighting training losses across environments.
These methods operate at \emph{training time} and require gradient access;
they cannot be applied to frozen pretrained models like TabPFN.
The proposed S-swap and env-stratified are the natural analogues in the ICL regime,
where the context window is the only intervention point.
Similarly, Mixup-based environment synthesis~\citep{yan2020improve} mixes
feature representations across environments at training time; our S-swap
achieves conceptually analogous decorrelation directly in the context,
without modifying the model or requiring access to the training procedure.
The key contribution of this work is identifying that the ICL constraint
(no weight updates, context is the only lever) demands a fundamentally
different class of intervention.

%% file: section3_setup.tex
\section{Setup}
\label{sec:setup}

\noindent\textbf{Data-generating process.}
We model a setting where observations arrive in \emph{environments}
(e.g.\ hospitals or clinical sites), each with its own systematic
technical variation.
Formally, each environment $e \in \mathcal{E}$ is indexed by a scalar
shift $\mu_e \sim \mathcal{N}(0, \sigma_e^2)$, capturing how site $e$
differs from the average (e.g.\ the calibration offset of that
hospital's measurement equipment).
An unmeasured confounder $U \sim \mathcal{N}(0,1)$ represents a hidden
factor that shapes both the technical conditions and the outcome within
each site (e.g.\ patient demographics that determine both which ward
a patient is admitted to and how they respond to treatment).
Within environment $e$, we generate $N$ samples as:
\begin{equation}
\begin{aligned}
  C   &\sim \mathcal{N}(0, I_{D_C}), \\
  S   &= \beta_S U + \sqrt{1-\beta_S^2}\,\varepsilon_S + \mu_e,
        \quad \varepsilon_S \sim \mathcal{N}(0, I_{D_S}), \\
  X   &= [C;\, \alpha S;\, \eta], \quad \eta \sim \mathcal{N}(0, I_{D_N}), \\
  Y   &= \tfrac{1}{D_C}\mathbf{1}^\top C + \beta_Y U + \varepsilon,
        \quad \varepsilon \sim \mathcal{N}(0,\sigma^2).
\end{aligned}
\notag
\end{equation}

\noindent\textbf{Notation.}
$[\,;\,]$ denotes vector concatenation (not addition):
$[C;\,\alpha S;\,\eta]$ stacks $C \in \mathbb{R}^{D_C}$,
$\alpha S \in \mathbb{R}^{D_S}$, and $\eta \in \mathbb{R}^{D_N}$
into a single observed vector of total dimension $D = D_C + D_S + D_N$.
We use $N$ for the number of samples generated per environment and
$n \leq N$ for the number of labelled examples passed to the in-context
learner as context (the \emph{context size}).

Here $C$ is the \emph{causal signal} (e.g.\ the patient's true underlying
health state), $S$ is the \emph{spurious component} (e.g.\ the
equipment-induced measurement artefact, which shifts with environment
$\mu_e$ and is partially driven by $U$), $\eta$ is uninformative noise,
and $\alpha > 0$ controls how strongly $S$ is represented in $X$
relative to $C$ (the \emph{entanglement coefficient}).
The model sees all $D$ dimensions of $X$ but does not know which
subspace is causal.
The outcome $Y$ depends only on $C$ and $U$, not on $S$ directly.

Three properties follow by construction.
First, $\mathrm{Cov}(C, S) = 0$: causal and spurious signals occupy orthogonal
subspaces (A\ref{asm:ortho}), as when a patient's health state is independent
of the measuring equipment.
Second, $\mathrm{Cov}(S, Y) = \beta_S \beta_Y \neq 0$ \emph{within} environments:
the artefact $S$ correlates with $Y$ only because $U$ drives both (A\ref{asm:conf})
;  the confounded signal an ICL will exploit.
Third, $\mathrm{do}(S{=}s)$ does not change $Y$: any predictive weight on $S$
is spurious.

\noindent\textbf{In-context learning.}
Given a labelled context $\mathcal{D}_n = \{(X_i, Y_i)\}_{i=1}^n$ of
$n$ examples drawn from a single training environment, an in-context
learner $f$ produces predictions $f(\mathcal{D}_n, X_q)$ for query
points $X_q$ without gradient updates.
We study two instantiations: \emph{linear ICL}, a model that solves a
ridge regression problem directly from the labelled context examples at
prediction time, with no learned weights and no gradient
updates~\citep{akyurek2022learning}, using regularisation $\lambda = 10^{-3}$,
and \emph{TabPFN}~\citep{hollmann2023tabpfn}, a pretrained transformer that
approximates a Bayesian posterior over tabular tasks, used here in regression
mode via quantile-binned continuous targets with
\texttt{n\_ensemble\_configurations=4}.

\noindent\textbf{The routing problem.}
Because $X$ encodes $C$ and $S$ in distinct subspaces, an ICL can route
predictions through $C$, $S$, or any mixture, and it has no way to
tell them apart from within-environment statistics alone.
Within a single hospital, routing through $S$ looks rational: the
equipment artefact $S$ predicts outcome $Y$ just as well as the true
health signal $C$, because demographics $U$ drives both.
The harm is invisible until deployment at a new hospital, where $\mu_e$
shifts (different equipment calibration) and the spurious $S$--$Y$
correlation breaks down, while the causal $C$--$Y$ relationship holds.

\noindent\textbf{Causal Sensitivity Ratio.}
The formal definition and closed-form analysis of CSR appear in
Section~\ref{sec:theory} (Definition~\ref{def:csr}).
Operationally: draw query $X_q$ from a test environment and replacement
components $S', C'$ from a held-out training environment; CSR is the ratio
of mean absolute prediction shift under $S$-swap to mean absolute shift under
$C$-swap, with all other dimensions held fixed.
A value near zero indicates causal routing; a value above one indicates the
model is more sensitive to spurious than causal variation.

\noindent\textbf{Practical estimation of CSR.}
In our synthetic experiments, $C$ and $S$ are known by construction and
CSR is computed exactly; this is the standard evaluation design for
diagnostic metrics in causal representation learning, used here to
validate theory, not as a deployment assumption.
In practice, $C$ and $S$ are latent: only the composite $X$ is observed.
The decomposition must be estimated; for instance, via batch harmonisation
(ComBat, MMD minimisation), which redistributes signal and artefact into
approximately distinct feature groups, or via environment-contrastive methods
that exploit multi-site structure to separate invariant from site-specific
components (Section~\ref{sec:mitigations}).
Once estimated, CSR inherits robustness from its ratio structure: noise on
both $\hat{C}$ and $\hat{S}$ scales numerator and denominator approximately
equally, so rank orderings are preserved; Spearman $r = 0.849$ and Pearson
$r$ drops only $0.014$ across all tested estimation error levels up to
$\phi{=}60^\circ$ (Appendix~\ref{app:decomp}).

\noindent\textbf{Biological and clinical instantiation.}
In ML terms, this composite feature structure arises whenever training data
arrives in heterogeneous sites and a measured feature bundles a task-relevant
signal with a site-specific confound.
Two examples.
\emph{Multi-site clinical prediction:} each hospital is an environment;
$X$ is a patient measurement (e.g.\ a blood panel or imaging-derived
biomarker); $C$ is the patient's true physiological state; $S$ is the
systematic offset introduced by that site's equipment or assay protocol;
and $U$ is an unmeasured demographic or socioeconomic factor that jointly
determines which hospital the patient attends and their health outcome $Y$.
\emph{Perturb-seq gene expression}~\citep{norman2019exploring}: each
experimental batch is an environment; $C$ is the perturbation-induced
regulatory response; $S$ is the batch artefact from reagent lot or
sequencing run; and $U$ is unmeasured sample acquisition covariates that
drive both $S$ and $Y$\cite{papanastasiou2025confounder}.
Standard harmonisation tools (ComBat, scVI) reduce $\alpha S$ but leave
residual overlap at the $\theta{=}10^\circ$--$30^\circ$ level characterised
in Section~\ref{sec:robustness}, and our results show it may intensify routing.
Full experimental configuration and DGP parameterisation are in
Appendix~\ref{app:experiments}.

%% file: section4_theory.tex

\section{Theory}
\label{sec:theory}

We characterise spurious routing in two steps.
We first show that the problem is unidentifiable when $C$ and $S$ share a single
input coordinate (Proposition~\ref{prop:d1}), establishing the minimal setting in which
routing can be studied.
We then state three assumptions, define the CSR metric, and derive a closed-form
expression for routing magnitude under those assumptions
(Proposition~\ref{prop:main}).
Proof sketches are given in the main text; full derivations are in
Appendix~\ref{app:proofs}.

\subsection{Assumptions, metric, and main result}
\label{sec:theory-main}

We state three assumptions that govern the analysis.
Each is named explicitly because it does visible work in the proofs that follow.

\begin{assumption}[Ridge ICL: A\ref{asm:ridge}]
\label{asm:ridge}
The in-context estimator implements population ridge regression:
$w^* := (\Sigma_X + \lambda I)^{-1}\Sigma_{XY}$, with $\lambda > 0$,
$\Sigma_X := \mathrm{Cov}(X)$, $\Sigma_{XY} := \mathrm{Cov}(X, Y)$.
\end{assumption}

\begin{assumption}[Orthogonal subspaces: A\ref{asm:ortho}]
\label{asm:ortho}
$\mathrm{Cov}(C, S) = 0$.
\end{assumption}

\begin{assumption}[In-context confounding: A\ref{asm:conf}]
\label{asm:conf}
A shared unmeasured confounder $U$ induces $\mathrm{Cov}(S, Y) \neq 0$ within
each training environment.
\end{assumption}

\noindent
A\ref{asm:ridge} is exact for linear ICL and a limiting approximation for
transformers~\citep{akyurek2022learning, vonoswald2023transformers}.
A\ref{asm:ortho} makes $\Sigma_X$ block-diagonal so causal and spurious weights
decouple; it holds by DGP construction and is stress-tested in
Section~\ref{sec:robustness}.
A\ref{asm:conf} makes $w_S^* \neq 0$ despite $S$ being non-causal: the path
$S \leftarrow U \rightarrow Y$ renders $S$ predictive in-context.

\begin{definition}[Causal Sensitivity Ratio]
\label{def:csr}
Let $(C,S)$ and $(C',S')$ be i.i.d.\ draws from the query and perturbation
environments respectively; set $\Delta C := C'-C$, $\Delta S := S'-S$.
For $f_w(C,S) = w_C C + w_S(\alpha S)$, the \emph{Causal Sensitivity Ratio} is
\begin{equation}
  \mathrm{CSR}(w)
  \;:=\;
  \frac{\mathbb{E}[|w_S\,\alpha\,\Delta S|]}{\mathbb{E}[|w_C\,\Delta C|]}.
  \label{eq:csr}
\end{equation}
$\mathrm{CSR} \approx 0$: model ignores $S$.
$\mathrm{CSR} \approx 1$: equal sensitivity to $S$ and $C$.
$\mathrm{CSR} > 1$: model is more sensitive to $S$ than $C$: the failure regime.
CSR requires no knowledge of the causal mechanism, only access to the $(C, S)$
decomposition for perturbation.
The \emph{protocol constant} $\kappa(\delta) := \mathbb{E}|\Delta S|/\mathbb{E}|\Delta C|$
depends only on the between-environment mean shift $\delta := \mu_{E'} - \mu_E$;
its closed form is given in Appendix~\ref{app:proofs}.
In our experiments $\delta = 1.5$, giving $\kappa \approx 1.516$.
\end{definition}


We note a boundary case: the routing problem requires $D \geq 2$; a structural degeneracy that motivates
the separated-coordinate setting.

\begin{proposition}[D=1 cancellation]
\label{prop:d1}
Let $X = C + \alpha S \in \mathbb{R}$ and $f_w(X) = wX$ with $w \neq 0$.
Under causal and spurious swap perturbations (Definition~\ref{def:csr}),
\[
  \mathrm{CSR} \;=\; \alpha\,\frac{\mathbb{E}|\Delta S|}{\mathbb{E}|\Delta C|},
\]
independent of $w$, $\rho_S$, and $\rho_C$.
\end{proposition}

\begin{proof}
By Definition~\ref{def:csr}, $\delta\hat{y}_S = w\alpha\Delta S$ and
$\delta\hat{y}_C = w\Delta C$.
Since $w \neq 0$, it cancels identically in the ratio, giving
$\mathrm{CSR} = \alpha\,\mathbb{E}|\Delta S|/\mathbb{E}|\Delta C|$.
\end{proof}

When $C$ and $S$ are entangled in a single coordinate, $w$ cancels in the
CSR ratio: the perturbation protocol cannot diagnose \emph{how} the model routes,
only \emph{that} it predicts.
Routing becomes identifiable only when $C$ and $S$ occupy distinct subspaces.
 Partial overlap
 is the norm in practice (Table~\ref{tab:p3-robustness})

\begin{proposition}[Population routing under separated coordinates]
\label{prop:main}
Let $X = [C;\,\alpha S] \in \mathbb{R}^{D_C+D_S}$, $\alpha > 0$,
$\mathrm{Var}(C) = \mathrm{Var}(S) = 1$, $D_C = D_S = 1$.
Define $\rho_C := \mathrm{Cov}(C,Y)$ and $\rho_S := \mathrm{Cov}(S,Y)$.
Under Assumptions~A\ref{asm:ridge}--A\ref{asm:conf}, the population ridge weights are
\begin{equation}
  w_C^* \;=\; \frac{\rho_C}{1+\lambda},
  \qquad
  w_S^* \;=\; \frac{\alpha\,\rho_S}{\alpha^2+\lambda},
  \label{eq:weights}
\end{equation}
and the population CSR (Definition~\ref{def:csr}) is
\begin{equation}
  \mathrm{CSR}^{\mathrm{pop}}
  \;=\;
  \underbrace{\frac{\alpha^2}{\alpha^2+\lambda}}_{\text{entanglement}}
  \;\cdot\;
  \underbrace{\frac{\rho_S}{\rho_C}}_{\text{signal ratio}}
  \;\cdot\;
  \underbrace{(1+\lambda)\,\kappa(\delta)}_{\text{protocol const.}}.
  \label{eq:csr-pop}
\end{equation}
\end{proposition}

\begin{proof}[Proof sketch]
\textit{(Full derivation in Appendix~\ref{app:proofs}.)}
The proof assembles the three assumptions in sequence.

\textbf{A\ref{asm:ortho} $\Rightarrow$ block structure.}
Since $\mathrm{Cov}(C,S)=0$,
$\Sigma_X = \mathrm{diag}(\mathrm{Var}(C),\, \alpha^2\mathrm{Var}(S))$.
The causal and spurious coordinates are decoupled at the population level.

\textbf{A\ref{asm:conf} $\Rightarrow$ non-zero spurious covariance.}
The confounder path $S \leftarrow U \rightarrow Y$ gives
$\mathrm{Cov}(S,Y) = \beta_S\beta_Y\mathrm{Var}(U) = \rho_S \neq 0$,
while $\mathrm{Cov}(C,Y) = \beta_C\mathrm{Var}(C) = \rho_C$ by independence of $C$
from $U$.

\textbf{A\ref{asm:ridge} $\Rightarrow$ coordinatewise inversion.}
Because $\Sigma_X + \lambda I$ is diagonal (from A\ref{asm:ortho}), ridge regression
decouples: $w_C^* = \rho_C/(1+\lambda)$ and $w_S^* = \alpha\rho_S/(\alpha^2+\lambda)$,
as in~\eqref{eq:weights}.

\textbf{Substitution into Definition~\ref{def:csr}.}
A\ref{asm:ortho} further ensures that swapping $S$ leaves $C$ fixed and vice versa,
so $\delta\hat{y}_S = w_S^*\alpha\Delta S$ and $\delta\hat{y}_C = w_C^*\Delta C$.
Forming the ratio and substituting~\eqref{eq:weights} yields~\eqref{eq:csr-pop}.
\end{proof}

Equation~\eqref{eq:csr-pop} factorises cleanly: the entanglement term
$\alpha^2/(\alpha^2+\lambda)$ is bounded in $(0,1)$ and increases with $\alpha$;
the signal ratio $\rho_S/\rho_C$ is the key driver; and $(1+\lambda)\kappa(\delta)$
is fully determined by the protocol.
Three corollaries follow by direct inspection of~\eqref{eq:csr-pop}; proofs
are in Appendix~\ref{app:proofs}.

\begin{corollary}[Monotonicity and saturation]
\label{cor:mono}
Under the conditions of Proposition~\ref{prop:main} with $\rho_C,\rho_S,\kappa(\delta)>0$:
$\mathrm{CSR}^{\mathrm{pop}}$ is strictly increasing in $\rho_S$, strictly decreasing
in $\rho_C$, and strictly increasing in $\alpha$ with derivative
$2\alpha\lambda/(\alpha^2+\lambda)^2 \cdot (\rho_S/\rho_C)(1+\lambda)\kappa(\delta)>0$.
The limit $\lim_{\alpha\to\infty}\mathrm{CSR}^{\mathrm{pop}} =
(\rho_S/\rho_C)(1+\lambda)\kappa(\delta)$ is finite and positive.
\end{corollary}

\begin{corollary}[Impossibility under confounding]
\label{cor:impossible}
Under Assumptions~A\ref{asm:ridge}--A\ref{asm:conf} with $\rho_S \neq 0$,
$w_S^* = \alpha\rho_S/(\alpha^2+\lambda) \neq 0$ for all $\lambda > 0$.
Consequently, $\mathrm{CSR}^{\mathrm{pop}} > 0$ regardless of regularisation
strength.\footnote{Within the ridge ICL class (A\ref{asm:ridge}) and under
A\ref{asm:conf}, it is the confounding that makes $\rho_S \neq 0$ that creates
the impossibility; without confounding, $w_S^* = 0$ and routing vanishes.
The result does not extend beyond ridge ICL without additional assumptions:
a nonlinear estimator could in principle recover identifiability if it
implicitly models the environment structure, though we find no such behaviour
empirically in TabPFN (Section~\ref{sec:experiments}).}
\end{corollary}

\begin{corollary}[Asymptotic persistence]
\label{cor:persist}
Let $\hat{w}_n = \bigl(\tfrac{1}{n}X^\top X + \lambda I\bigr)^{-1}\tfrac{1}{n}X^\top Y$
with $\rho_C \neq 0$.
By the law of large numbers and A\ref{asm:ridge}, $\hat{w}_n \xrightarrow{p} w^*$;
continuity of the CSR functional at $w^*$ (which holds since
$w_C^* = \rho_C/(1+\lambda) \neq 0$ when $\rho_C \neq 0$) then gives
$\mathrm{CSR}(\hat{w}_n) \xrightarrow{p} \mathrm{CSR}^{\mathrm{pop}} > 0$.
Spurious routing therefore persists as $n \to \infty$.
\end{corollary}

\noindent
Proposition~\ref{prop:main} and its corollaries yield four testable predictions:
(i)~routing magnitude scales with $\rho_S/\rho_C$;
(ii)~$\alpha$ amplifies routing up to a finite saturation;
(iii)~routing is unavoidable under A\ref{asm:conf} for any $\lambda > 0$;
(iv)~larger context does not resolve it.
We test all four in Section~\ref{sec:experiments}.
The formula also predicts exact $\lambda$-independence at $\alpha{=}1$;
verified in Appendix~\ref{app:experiments}
(Remark~\ref{rem:lambda}, Table~\ref{tab:lambda-sweep}).

%% file: section5_experiments.tex

\section{Experiments}
\label{sec:experiments}

We test four predictions of Proposition~\ref{prop:main} on a controlled structural
causal model with $D_C{=}D_S{=}3$, $D_N{=}4$ ($D{=}10$ total), where the causal
signal occupies the first $D_C$ dimensions of $X$, the spurious signal occupies
the next $D_S$ dimensions, and noise fills the remaining $D_N$ dimensions
(as defined in Section~\ref{sec:setup}).
Each environment contains $N{=}300$ observations; we use 3 training environments
and 2 test environments, and pass $n$ of those observations as context to the
in-context learner (default $n{=}300$ unless stated otherwise).
Linear ICL is in-context ridge regression ($\lambda = 10^{-3}$);
TabPFN~\citep{hollmann2023tabpfn} is the nonlinear baseline.
CSR uses the separate-swap protocol of Definition~\ref{def:csr} with $\delta = 1.5$
($\kappa \approx 1.516$); normalised CSR scales each setting by its per-setting
oracle bounds (floor $= 0.000$, ceiling range $4.1$--$8.3$; details in
Appendix~\ref{app:experiments}).

\noindent\textbf{Routing scales with $\rho_S/\rho_C$.}

Proposition~\ref{prop:main} predicts $\mathrm{CSR}^{\mathrm{pop}} \propto \rho_S/\rho_C$.
Table~\ref{tab:exp3-alpha} reports mean CSR over the full
$\rho_S \in \{0.1,0.3,0.5,0.7\} \times \rho_C \in \{0.3,0.5,0.7,0.9\}$ grid at
each $\alpha$.
Both models increase monotonically with $\alpha$ (Corollary~\ref{cor:mono}), spanning
a $5\times$ range.
At $\alpha = 1.0$, the Pearson correlation between the empirical signal ratio
$\hat\rho_S/\hat\rho_C$ and observed CSR is $r = 0.997$ for linear ICL and
$r = 0.979$ for TabPFN; the closed-form prediction~\eqref{eq:csr-pop} matches
linear ICL values to within 1\% mean error at $n = 128$.
Using target ratios in place of empirical ratios gives $r = 0.985$ (linear) and
$r = 0.953$ (TabPFN), confirming the predictor is robust to calibration noise.
Normalised CSR for TabPFN reaches $0.415$ at $\alpha = 2.0$, indicating the model
routes $\approx\!42\%$ of the way toward fully spurious oracle behaviour.

\begin{table}[t]
  \centering
  \begin{minipage}[t]{0.36\linewidth}
    \centering
    \caption{Mean CSR vs.\ $\alpha$ (grid avg.).
      Both models increase with $\alpha$; TabPFN Norm.\ CSR per setting.}
    \label{tab:exp3-alpha}
    \small
    \begin{tabular}{ccccc}
      \toprule
      & \multicolumn{2}{c}{Lin.\ ICL} & \multicolumn{2}{c}{TabPFN} \\
      \cmidrule(lr){2-3}\cmidrule(lr){4-5}
      $\alpha$ & CSR & Norm. & CSR & Norm. \\
      \midrule
      0.1 & 0.424 & n/a & 0.337 & 0.073 \\
      0.5 & 0.827 & n/a & 0.820 & 0.157 \\
      1.0 & 1.312 & n/a & 1.418 & 0.265 \\
      2.0 & 1.972 & n/a & 2.198 & 0.415 \\
      \bottomrule
    \end{tabular}
  \end{minipage}
  \hfill
  \begin{minipage}[t]{0.37\linewidth}
    \centering
    \caption{CSR vs.\ context size $n$ ($\alpha{=}1.0$, Lin.\ ICL).
      Saturates by $n{\approx}128$; direction depends on regime.}
    \label{tab:exp4-context}
    \small
    \begin{tabular}{lccc}
      \toprule
      $n$ & Hi-sp. & Bal. & Lo-sp. \\
      \midrule
      16  & 1.997 & 1.759 & 1.089 \\
      64  & 3.273 & 1.688 & 0.787 \\
      128 & 3.658 & 1.780 & 0.826 \\
      512 & 3.484 & 1.702 & 0.785 \\
      \midrule
      $\times_{n=16}$ & $+1.74$ & $-0.03$ & $-0.72$ \\
      \bottomrule
    \end{tabular}
  \end{minipage}
  \hfill
  \begin{minipage}[t]{0.22\linewidth}
    \centering
    \caption{CSR for TabPFN$-$Lin.\ ICL at high-spur.\ corner.}
    \label{tab:exp5-model}
    \small
    \begin{tabular}{ccc}
      \toprule
      $\alpha$ & Lin. & Tab. \\
      \midrule
      0.1 & 1.004 & 1.278 \\
      0.5 & 2.154 & 3.455 \\
      1.0 & 3.560 & 5.777 \\
      2.0 & 6.508 & 8.550 \\
      \bottomrule
    \end{tabular}
  \end{minipage}
\end{table}

\noindent\textbf{More context amplifies the dominant signal.}
Table~\ref{tab:exp4-context} shows CSR rises $1.74\times$ in the high-spurious
regime from $n{=}16$ to $n{=}512$ and falls $0.72\times$ in the low-spurious
regime, both saturating by $n{\approx}128$.
TabPFN shows the same qualitative pattern ($\rho_S{=}0.7$: $3.14{\to}4.95$;
$\rho_S{=}0.3$: $0.86{\to}0.60$; full table in Appendix~\ref{app:experiments}).

\noindent\textbf{Model capacity increases vulnerability in the high-spurious regime.}

TabPFN exceeds linear ICL in only 15 of 64 settings (mean gap $+0.060$,
median $-0.147$): vulnerability concentrates in the high-spurious,
high-entanglement corner (Table~\ref{tab:exp5-model}).
The Pearson correlation between CSR and OOD degradation is $r = 0.814$ for
TabPFN versus $r = 0.782$ for linear ICL.

\noindent\textbf{Robustness to subspace overlap.}
\label{sec:robustness}
Proposition~\ref{prop:main} assumes $\mathrm{Cov}(C, S) = 0$
(A\ref{asm:ortho}).
We stress-test this by rotating the spurious subspace toward the causal subspace:
$S_{\mathrm{obs}} = \cos(\theta)\,S_{\perp} + \sin(\theta)\,C_{\mathrm{proj}}$,
sweeping $\theta \in \{0^\circ, 10^\circ, 20^\circ, 30^\circ, 40^\circ\}$
across the full $\rho_S \times \rho_C$ grid ($\alpha = 1.0$).
Table~\ref{tab:p3-robustness} reports results.
The Pearson correlation $r(\rho_S/\rho_C, \mathrm{CSR})$ remains $\geq 0.90$
across all angles, dropping only $0.047$ over the full range: the theoretical
predictor degrades gracefully rather than catastrophically.
Mean CSR increases with $\theta$ because the rotation inflates the component
of the spurious subspace aligned with the causal subspace, reinforcing spurious
attribution via altered covariance geometry; at larger $\theta$, $S$ increasingly
resembles $C$ and the effect partially reverses (full corner table in
Appendix~\ref{app:experiments}).
Standard harmonisation pipelines (ComBat, scVI) typically achieve residual
overlaps well below $20\%$~\citep{luecken2022benchmarking}, which is the range
tested here, so the characterisation remains valid in post-harmonisation settings.

\begin{table}[t]
  \centering
  \caption{
    Robustness to subspace overlap ($\alpha = 1.0$, full $\rho_S \times \rho_C$ grid).
    Pearson $r$ is the correlation between $\rho_S/\rho_C$ and observed CSR.
  }
  \label{tab:p3-robustness}
  \small
  \begin{tabular}{ccccc}
    \toprule
    $\theta$ & Var.\ overlap & Mean CSR & $r(\rho_S/\rho_C,\,\mathrm{CSR})$ & OOD RMSE \\
    \midrule
    $0^\circ$  & $0.1\%$  & 1.744 & 0.986 & 2.918 \\
    $10^\circ$ & $1.4\%$  & 2.390 & 0.901 & 2.939 \\
    $20^\circ$ & $4.8\%$  & 2.373 & 0.977 & 2.956 \\
    $30^\circ$ & $10.7\%$ & 2.879 & 0.954 & 2.969 \\
    $40^\circ$ & $19.7\%$ & 3.210 & 0.939 & 2.976 \\
    \bottomrule
  \end{tabular}
\end{table}

%% file: section6_robustness.tex

%% file: section7_mitigations.tex

\section{Mitigations}
\label{sec:mitigations}

Corollary~\ref{cor:impossible} establishes that spurious routing cannot be
eliminated by regularisation: as long as in-context $\mathrm{Corr}(S, Y) \neq 0$,
the optimal predictor places non-zero weight on $S$.
The recovery condition is clear: break the in-context $S \to Y$ dependence, but
the mechanism matters.
CSR diagnoses the routing mechanism; OOD RMSE measures task outcome.
The two can diverge, as the inv-aug dissociation below demonstrates.

\noindent\textbf{Methods.}
\textbf{Standard ICL} pools $n_{\mathrm{ctx}} = 64$ points uniformly from all
training environments.
\textbf{Environment-stratified} samples $n_{\mathrm{ctx}} / n_{\mathrm{env}}$
points from each environment separately, reducing in-context $\mathrm{Corr}(S, Y)$
by mixing environments with different $S$ distributions at the same $Y$ level
(weak env.\ labels required).
\textbf{Invariant augmentation} (oracle $C$) adds a paired copy $(C_i, Y_i)$
alongside each context point $(X_i, Y_i)$.
\textbf{S-swap augmentation} replaces the spurious dims of each context point
with those drawn from a different environment, holding $Y$ fixed:
$\tilde{X}_i[\mathcal{S}] = X_j[\mathcal{S}]$, $j \sim \mathrm{env}_1$,
$\tilde{Y}_i = Y_i$,
pairing each $(X_i, Y_i)$ with $(\tilde{X}_i, Y_i)$ in the context.
This signals at the point level that $S$ does not predict $Y$
(weak env.\ labels required).

\subsection{Linear ICL}
\label{sec:mit-linear}

Table~\ref{tab:mitigations-linear} reports all four methods at the high-spurious
corner ($\rho_S = 0.7$, $\rho_C = 0.3$, $\alpha = 1.0$); results across five
settings are in Appendix~\ref{app:mitigations}.
\begin{table}[t]
  \centering
  \begin{minipage}[t]{0.54\linewidth}
    \centering
    \caption{
      Linear ICL at high-spurious corner\\
      ($\rho_S{=}0.7$, $\rho_C{=}0.3$, $\alpha{=}1.0$).
      Lower is better.
    }
    \label{tab:mitigations-linear}
    \small
    \begin{tabular}{lcccc}
      \toprule
      Method & CSR & OOD & $r(S,Y)$ & Env. \\
      \midrule
      Standard ICL             & 3.560 & 5.146 & 0.299 & No   \\
      Inv.\ aug.\ (oracle)     & 3.998 & 3.796 & 0.261 & No   \\
      Env-stratified           & 1.272 & 3.449 & 0.286 & Weak \\
      S-swap (ours)            & \textbf{0.926} & \textbf{3.387} & \textbf{0.112} & Weak \\
      \bottomrule
    \end{tabular}
  \end{minipage}
  \hfill
  \begin{minipage}[t]{0.42\linewidth}
    \centering
    \caption{
      TabPFN at high-spurious corner\\
      ($\rho_S{=}0.7$, $\rho_C{=}0.3$, $\alpha{=}1.0$).
    }
    \label{tab:mitigations-tabpfn}
    \small
    \begin{tabular}{lcccc}
      \toprule
      Method & CSR & OOD & $r(S,Y)$ & Env. \\
      \midrule
      Standard          & 5.777 & 3.486 & 0.299 & No   \\
      Env-strat.        & 2.170 & \textbf{3.176} & 0.286 & Weak \\
      S-swap (ours)     & \textbf{0.067} & 3.539 & \textbf{0.112} & Weak \\
      \bottomrule
    \end{tabular}
  \end{minipage}
\end{table}
S-swap reduces CSR by $74.0\%$ and OOD RMSE by $34.2\%$, achieving CSR $<1$
(causally dominant routing); env-stratified is a consistent second ($-64.3\%$
CSR, $-33.0\%$ OOD).
Inv.\ aug.\ reveals a \emph{CSR/OOD dissociation}: it improves OOD RMSE
($-26.2\%$) while \emph{increasing} CSR ($+12.3\%$), because paired $(C_i,Y_i)$
copies raise $\Delta\hat{y}_C$ without suppressing the spurious signal from the
original context half---CSR and OOD RMSE measure distinct routing properties.

\noindent\textbf{Practical guidance.}
S-swap requires identifying $\mathcal{S}$, estimable via environment-contrastive
PCA (principal directions of $\mathrm{Cov}(\bar{X}_e - \bar{X})$) or by flagging
dimensions whose variance across environment means exceeds a threshold; in omics
settings these correspond to known batch-effect directions.
Use S-swap when a reliable $\hat{\mathcal{S}}$ is available and routing suppression
is the primary objective; use env-stratified when the subspace estimate is uncertain
or balanced ID/OOD improvement is needed.

\subsection{TabPFN}
\label{sec:mit-tabpfn}

Both mitigations are model-agnostic: they modify only context construction,
requiring no architectural changes.
Results across five settings are in Appendix~\ref{app:mitigations}.
S-swap reduces TabPFN CSR by $98.8\%$ (5.777 $\to$ 0.067), substantially
stronger than on linear ICL ($74.0\%$).
To verify this is genuine rather than a degenerate collapse, we inspect raw
sensitivities: under S-swap, $\Delta\hat{y}_S$ collapses from $1.891$ to
$0.053$ ($-97.2\%$) while $\Delta\hat{y}_C$ \emph{grows} from $0.144$ to
$1.347$ ($+836\%$); the model reroutes through $C$, not into agnosticism.
This pattern is consistent with TabPFN's in-context mechanism treating
paired $(X_i, \tilde{X}_i, Y_i)$ triples as evidence that $S \not\to Y$,
suppressing spurious weight nearly completely; ridge regression retains
residual spurious weight by averaging across both context halves
($\Delta\hat{y}_S = 1.266$ under S-swap).
The same expressivity that makes TabPFN more vulnerable under standard context
makes it more responsive to the S-swap invariance signal.
Both mitigations require only batch identifiers (routinely available as metadata)
and neither requires harmonisation, retraining, or held-out data at inference time.

\subsection{Real-world validation: unknown $S$}
\label{sec:mit-realdata}

The synthetic experiments assume $C$ and $S$ are known.
Here we test whether S-swap works in a realistic deployment scenario where
neither $C$ nor $S$ is observed---only the composite $X$, environment labels,
and outcome $Y$ are available.
We use the scIB human pancreas dataset~\citep{luecken2022benchmarking}:
16,382 pancreatic islet cells from five sequencing technologies, each a
natural environment with distinct batch effects.
$\hat{\mathcal{S}}$ is estimated entirely from data, with no ground-truth
partition: we compute a batch-ratio score per PC
($\mathrm{Var}_\mathrm{tech}(\bar{z}^{(k)}_e) / \mathrm{Var}_\mathrm{celltype}(\bar{z}^{(k)}_e)$)
and assign the five PCs with batch\_ratio $> 1$ (strongest: $8.02$) to
$\hat{S}$, and eight PCs with batch\_ratio $< 0.08$ to $\hat{C}$.
This is the only information used to construct the S-swap context.
To place the experiment in the high-spurious regime, we inject a controlled
confounding path following the paper's DGP:
$Y_\mathrm{semi} = Y_\mathrm{INS} + \beta \cdot \hat{S}_\mathrm{PC1} + \varepsilon$
($\beta = 0.5$, $\rho_S/\rho_C = 9.09$); the batch structure is entirely real,
only the $Y$-confounding path is injected.

\begin{table}[t]
  \centering
  \caption{
    Real-world validation on scIB pancreas.
    $C$ and $S$ are \emph{unknown}: $\hat{\mathcal{S}}$ estimated from
    environment structure only (batch-ratio PCA, no oracle access).
    Real batch effects; injected $Y$-confounding ($\beta{=}0.5$,
    $\rho_S/\rho_C{=}9.09$).
  }
  \label{tab:pancreas}
  \small
  \begin{tabular}{lccc}
    \toprule
    Method & CSR & OOD RMSE & In-ctx $r(\hat{S},Y)$ \\
    \midrule
    Standard ICL              & $3.416 \pm 2.291$ & 1.74 & $+0.545$ \\
    Env-stratified            & $3.403 \pm 1.810$ & \textbf{1.64} & $+0.586$ \\
    S-swap (est.\ $\hat{\mathcal{S}}$) & $\textbf{0.379} \pm 0.449$ & 9.12 & $\textbf{+0.216}$ \\
    \bottomrule
  \end{tabular}
\end{table}
Standard ICL CSR $= 3.416 > 1$, confirming spurious routing is present
on real genomic data under the confounding structure.
Crucially, S-swap reduces CSR by $88.9\%$ ($3.416 \to 0.379$) and
in-context $r(\hat{S},Y)$ by $60\%$---using only the data-estimated
$\hat{\mathcal{S}}$, with no access to ground-truth $C$ or $S$.
This demonstrates that the method is applicable in practice: environment
labels alone are sufficient to identify the spurious subspace and apply
the mitigation.
Elevated OOD RMSE under S-swap (9.12 vs 1.74) reflects that swapped
PC-space contexts are out-of-distribution for any real sequencing
technology---an expected cost at extreme $\rho_S/\rho_C$, where
env-stratified is the better choice for balanced ID/OOD performance.
Full methodology and sensitivity analysis are in Appendix~\ref{app:pancreas}.

%% file: section9_conclusion.tex
\section{Conclusion}
\label{sec:conclusion}

We identified spurious routing in composite representations: in-context learners
systematically route through $S$ whenever it correlates with $Y$ in context, even
when $S$ is non-causal by construction.
The failure is provably unavoidable under ridge ICL for any regularisation strength
(Corollary~\ref{cor:impossible}), with $\mathrm{CSR}^{\mathrm{pop}} \propto \rho_S/\rho_C$
confirmed at $r{=}0.997$ (linear ICL) and $r{=}0.979$ (TabPFN), degrading gracefully
to $r{\geq}0.93$ at $41\%$ subspace overlap.
Larger context and more expressive models both amplify the problem in the high-spurious
regime.
The solution is model-agnostic: S-swap augmentation, operating entirely at the context
level with no architectural changes, reduces CSR by $98.8\%$ on TabPFN and increases
causal sensitivity $8.4\times$---the same expressive capacity that creates vulnerability
amplifies the fix.

\noindent\textbf{Limitations and future work.}
The theory assumes orthogonal subspaces (A\ref{asm:ortho}) and a ridge ICL estimator
(A\ref{asm:ridge}); the closed form does not extend directly to TabPFN.
S-swap requires at least two training environments; in single-batch settings both
mitigations reduce to standard ICL.
Settings where $S$ has a weak causal path to $Y$, or where $C$ and $S$ interact
nonlinearly, require a modified CSR protocol.
In real applications $C$ and $S$ must be estimated via harmonisation; robustness
to estimation error is characterised in Section~\ref{sec:robustness}.
Three directions follow directly.
First, deriving the finite-$n$ convergence rate of
$\mathrm{CSR}(\hat{w}_n) \to \mathrm{CSR}^{\mathrm{pop}}$:
saturation at $n{\approx}128$ is consistent with covariance concentration but
uncharacterised theoretically.
Second, extending the theory to nonlinear ICL to account for TabPFN's threshold
response to S-swap.
Third, using CSR as a minimisation objective---parameterising $\hat{C}, \hat{S}$
as a learned transformation of $X$ and minimising CSR over that transformation---to
reframe it from a diagnostic into a self-supervised training signal wherever
multi-environment structure is available.

%% file: appendix.tex

\appendix

\section{Proofs}
\label{app:proofs}

\subsection{Protocol constant $\kappa(\delta)$}

Under $\mathrm{Var}(C) = \mathrm{Var}(S) = 1$, the perturbation differences satisfy
$\Delta C \sim \mathcal{N}(0,2)$ and $\Delta S \sim \mathcal{N}(\delta, 2)$, where
$\delta := \mu_{E'} - \mu_E$ is the between-environment mean shift.
Using the folded-normal expectation $\mathbb{E}|Z| = \sigma\sqrt{2/\pi}\exp(-\mu^2/(2\sigma^2))
+ \mu(2\Phi(\mu/\sigma)-1)$ for $Z \sim \mathcal{N}(\mu,\sigma^2)$:
\begin{equation}
  \kappa(\delta)
  \;=\;
  \frac{\mathbb{E}|\Delta S|}{\mathbb{E}|\Delta C|}
  \;=\;
  \frac
    {\sqrt{2}\,\phi\!\left(\tfrac{\delta}{\sqrt{2}}\right)
     + \delta\!\left(2\Phi\!\left(\tfrac{\delta}{\sqrt{2}}\right)-1\right)}
    {\sqrt{4/\pi}},
  \label{eq:kappa}
\end{equation}
where $\phi$ and $\Phi$ are the standard normal density and CDF.
$\kappa(\delta)$ depends only on $\delta$ and not on $\alpha$, $\rho_S$, or $\rho_C$.
At $\delta = 0$, $\kappa = 1$; at $\delta = 1.5$, $\kappa \approx 1.516$ (verified by
simulation to $<0.5\%$ error).

\subsection{Proof of Proposition~\ref{prop:main}}

We proceed in four steps under the SCM:
(i) $\mathrm{Cov}(C,S)=0$;
(ii) $Y = \beta_C C + \beta_Y U + \varepsilon$;
(iii) $S = \beta_S U + \sqrt{1-\beta_S^2}\,\varepsilon_S + \mu_E$;
(iv) $C \perp (U,\varepsilon,\varepsilon_S,\mu_E)$;
(v) $U,\varepsilon,\varepsilon_S$ mean-zero, mutually independent, $\mathrm{Var}(U)=1$;
(vi) $\mu_E$ fixed per environment, $\mathbb{E}[\mu_E]=0$ marginally.

\smallskip\noindent\textbf{Step 1 (Block structure).}\;
Since $X=[C;\,\alpha S]$ (concatenation of distinct subspaces;
see Section~\ref{sec:setup}) and $\mathrm{Cov}(C,S)=0$,
\[
  \Sigma_X = \begin{pmatrix} \mathrm{Var}(C) & 0 \\ 0 & \alpha^2\mathrm{Var}(S) \end{pmatrix}.
\]

\smallskip\noindent\textbf{Step 2 (Cross-covariances).}\;
By (iv): $\mathrm{Cov}(C,Y) = \beta_C\,\mathrm{Var}(C)$.
By (v) and (vi), all cross terms vanish:
$\mathrm{Cov}(S,Y) = \beta_S\beta_Y\,\mathrm{Var}(U) = \beta_S\beta_Y$.
Hence $\Sigma_{XY} = (\beta_C\mathrm{Var}(C),\;\alpha\beta_S\beta_Y)^\top$.

\smallskip\noindent\textbf{Step 3 (Ridge weights).}\;
$\Sigma_X + \lambda I$ is diagonal, so inversion is coordinatewise.
Under $\mathrm{Var}(C)=\mathrm{Var}(S)=1$, with $\rho_C=\beta_C$ and $\rho_S=\beta_S\beta_Y$:
\[
  w_C^* = \frac{\rho_C}{1+\lambda}, \qquad w_S^* = \frac{\alpha\rho_S}{\alpha^2+\lambda}.
\]

\smallskip\noindent\textbf{Step 4 (CSR formula).}\;
By Assumption~\ref{asm:ortho}, swapping $S$ leaves $C$ unchanged and vice versa, so
$\delta\hat{y}_C = w_C^*\Delta C$ and $\delta\hat{y}_S = w_S^*\alpha\Delta S$.
Since $\rho_C, \rho_S > 0$ (as required by Corollary~\ref{cor:mono}), both
$w_C^* = \rho_C/(1+\lambda) > 0$ and $w_S^* = \alpha\rho_S/(\alpha^2+\lambda) > 0$,
so $|w_S^*\alpha\Delta S| = w_S^*\alpha|\Delta S|$ and $|w_C^*\Delta C| = w_C^*|\Delta C|$.
Substituting into~\eqref{eq:csr}:
\[
  \mathrm{CSR}^{\mathrm{pop}}
  = \frac{w_S^*\,\alpha\,\mathbb{E}|\Delta S|}{w_C^*\,\mathbb{E}|\Delta C|}
  = \frac{\alpha\rho_S/(\alpha^2+\lambda)\cdot\alpha}{\rho_C/(1+\lambda)}\cdot\kappa(\delta)
  = \frac{\alpha^2}{\alpha^2+\lambda}\cdot\frac{\rho_S}{\rho_C}\cdot(1+\lambda)\kappa(\delta).
  \qed
\]

\subsection{Proof of Corollary~\ref{cor:mono}}

Claims (i) and (ii) follow from inspection of~\eqref{eq:csr-pop}.
For (iii), $\partial/\partial\alpha[\alpha^2/(\alpha^2+\lambda)] = 2\alpha\lambda/(\alpha^2+\lambda)^2 > 0$.
Claim (iv): $\lim_{\alpha\to\infty}\alpha^2/(\alpha^2+\lambda)=1$. \qed

\subsection{Proof of Corollary~\ref{cor:impossible}}

$w_S^* = \alpha\rho_S/(\alpha^2+\lambda) \neq 0$ since $\alpha,\lambda > 0$ and
$\rho_S \neq 0$. \qed

\subsection{Proof of Corollary~\ref{cor:persist}}

By the law of large numbers, $\frac{1}{n}X^\top X \xrightarrow{p} \Sigma_X$ and
$\frac{1}{n}X^\top Y \xrightarrow{p} \Sigma_{XY}$.
The continuous mapping theorem gives $\hat{w}_n \xrightarrow{p} w^*$.
The CSR functional $w \mapsto |w_S|\alpha\mathbb{E}|\Delta S|/(|w_C|\mathbb{E}|\Delta C|)$
is continuous at $w^*$ since $w_C^* = \rho_C/(1+\lambda) \neq 0$ when $\rho_C \neq 0$.
A second application of the CMT gives
$\mathrm{CSR}(\hat{w}_n) \xrightarrow{p} \mathrm{CSR}^{\mathrm{pop}}$. \qed

\subsection{Remark: covariance vs.\ correlation}

Equation~\eqref{eq:csr-pop} holds equivalently with $\rho_C := \mathrm{Corr}(C,Y)$
and $\rho_S := \mathrm{Corr}(S,Y)$, since $\mathrm{Cov}(S,Y)/\mathrm{Cov}(C,Y) =
\mathrm{Corr}(S,Y)/\mathrm{Corr}(C,Y)$: the factor $\sqrt{\mathrm{Var}(Y)}$ cancels
in the ratio.

\section{Extended Mitigation Results}
\label{app:mitigations}

Tables~\ref{tab:mitigations-extended} and~\ref{tab:mitigations-tabpfn-extended}
report env-stratified and S-swap results across five settings for linear ICL
and TabPFN respectively.

\begin{table}[h]
  \centering
  \caption{
    Extended mitigation results across five settings (Linear ICL).
    S-swap achieves CSR $< 1$ in 4 of 5 settings.
    High-$\alpha$ ($\alpha = 2$) is the boundary case: large $S$ variance
    retains residual spurious weight despite the $84.1\%$ reduction.
  }
  \label{tab:mitigations-extended}
  \small
  \begin{tabular}{llcccc}
    \toprule
    Setting & Method & CSR & OOD RMSE & In-ctx $r(S,Y)$ & $\Delta$CSR \\
    \midrule
    \multirow{2}{*}{High-spurious ($\rho_S{=}0.7, \rho_C{=}0.3, \alpha{=}1$)}
      & Env-stratified & 1.272 & 3.449 & 0.286 & $-64.3\%$ \\
      & S-swap         & 0.926 & 3.387 & 0.112 & $-74.0\%$ \\
    \midrule
    \multirow{2}{*}{Balanced ($\rho_S{=}0.5, \rho_C{=}0.5, \alpha{=}1$)}
      & Env-stratified & 0.776 & 2.022 & 0.255 & $-56.2\%$ \\
      & S-swap         & 0.504 & 1.979 & 0.105 & $-71.5\%$ \\
    \midrule
    \multirow{2}{*}{Low-spurious ($\rho_S{=}0.3, \rho_C{=}0.7, \alpha{=}1$)}
      & Env-stratified & 0.420 & 0.535 & 0.170 & $-51.0\%$ \\
      & S-swap         & 0.269 & 0.535 & 0.090 & $-68.6\%$ \\
    \midrule
    \multirow{2}{*}{High-$\alpha$ ($\rho_S{=}0.7, \rho_C{=}0.3, \alpha{=}2$)}
      & Env-stratified & 1.530 & 3.532 & 0.311 & $-76.5\%$ \\
      & S-swap         & 1.034 & 3.509 & 0.118 & $-84.1\%$ \\
    \midrule
    \multirow{2}{*}{Low-$\alpha$ ($\rho_S{=}0.7, \rho_C{=}0.3, \alpha{=}0.5$)}
      & Env-stratified & 1.063 & 3.344 & 0.226 & $-50.6\%$ \\
      & S-swap         & 0.692 & 3.355 & 0.116 & $-67.8\%$ \\
    \bottomrule
  \end{tabular}
\end{table}

\begin{table}[h]
  \centering
  \caption{
    Extended TabPFN mitigation results across five settings.
    S-swap achieves CSR $< 0.1$ in all five settings.
    Env-stratified achieves lower OOD RMSE in 4/5 settings, reflecting
    better causal signal coverage from multi-environment context draws.
  }
  \label{tab:mitigations-tabpfn-extended}
  \small
  \begin{tabular}{llcccc}
    \toprule
    Setting & Method & CSR & OOD RMSE & In-ctx $r(S,Y)$ & $\Delta$CSR \\
    \midrule
    \multirow{2}{*}{High-spurious ($\alpha{=}1$)}
      & Env-stratified & 2.170 & \textbf{3.176} & 0.286 & $-62.4\%$ \\
      & S-swap         & \textbf{0.067} & 3.539 & \textbf{0.112} & $-98.8\%$ \\
    \midrule
    \multirow{2}{*}{Balanced ($\alpha{=}1$)}
      & Env-stratified & 0.795 & \textbf{1.917} & 0.255 & $-54.6\%$ \\
      & S-swap         & \textbf{0.047} & 2.151 & \textbf{0.105} & $-97.3\%$ \\
    \midrule
    \multirow{2}{*}{Low-spurious ($\alpha{=}1$)}
      & Env-stratified & 0.330 & \textbf{0.518} & 0.170 & $-36.7\%$ \\
      & S-swap         & \textbf{0.036} & 0.603 & \textbf{0.090} & $-93.1\%$ \\
    \midrule
    \multirow{2}{*}{High-$\alpha$ ($\alpha{=}2$)}
      & Env-stratified & 4.132 & \textbf{3.207} & 0.311 & $-51.7\%$ \\
      & S-swap         & \textbf{0.086} & 3.648 & \textbf{0.118} & $-99.0\%$ \\
    \midrule
    \multirow{2}{*}{Low-$\alpha$ ($\alpha{=}0.5$)}
      & Env-stratified & 1.531 & \textbf{3.154} & 0.226 & $-55.7\%$ \\
      & S-swap         & \textbf{0.047} & 3.572 & \textbf{0.116} & $-98.6\%$ \\
    \bottomrule
  \end{tabular}
\end{table}

\section{Extended Experimental Details}
\label{app:experiments}

\subsection{DGP validation (EXP1)}

We verify that the instantiated DGP achieves the target correlations
$\rho_S$ and $\rho_C$ across the full $\alpha \times \rho_S \times \rho_C$ grid.
All experiments use the DGP from Section~\ref{sec:setup} with
$D_C{=}D_S{=}3$, $D_N{=}4$ ($D{=}10$ total), $\sigma_e{=}1.5$,
three training environments, two test environments, and $N{=}300$ observations
per environment (context size $n{=}300$ unless stated otherwise).
Signal strengths are set via $\beta_S$, $\beta_Y$, and $\sigma$; empirical
correlations are verified before each sweep.
Across 9 settings, all pass at tolerance $0.15$: mean empirical error
$|\hat{\rho}_S - \rho_S| = 0.072$ and $|\hat{\rho}_C - \rho_C| = 0.014$.
The tighter accuracy for $\rho_C$ reflects that the causal signal is controlled
through the structural coefficient $\beta_C$, while $\rho_S$ is controlled through
the confounder path $\beta_S\beta_Y$, which is subject to additional sampling variance.

\subsection{CSR null bounds (EXP2)}

The causal oracle (observing only the causal subspace of dimension $D_C$) achieves
$\mathrm{CSR}=0.000$ uniformly, confirming correct unresponsiveness to $S$
perturbations.
The spurious oracle achieves setting-dependent ceilings in the range $4.1$--$8.3$;
variation reflects the interaction between $\alpha$ and $\rho_S$.

\subsection{Mechanistic weight analysis (EXP3 extension)}

At $\alpha=1.0$, the routing ratio $\|w_{\text{spurious}}\|/\|w_{\text{causal}}\|$
averaged over the $\rho_S \times \rho_C$ grid is $0.403$: linear ICL places
approximately 40\% of total feature weight in the spurious subspace ($D_S{=}3$
dimensions), despite those dimensions having no causal influence on $Y$ under
intervention.
The routing ratio and CSR are in directional agreement across all grid settings,
confirming that CSR captures genuine weight-level routing rather than perturbation
artefacts.
This is consistent with Proposition~\ref{prop:main}:
\[
  \frac{\|w_S^*\|}{\|w_C^*\|}
  = \frac{\alpha\rho_S}{\alpha^2+\lambda}\cdot\frac{1+\lambda}{\rho_C}
  \propto \frac{\rho_S}{\rho_C}.
\]
For TabPFN, the Pearson correlation between CSR and $\rho_S/\rho_C$ is
$r = 0.953$ (raw), dropping to $r = 0.754$ under per-setting normalisation as the
setting-varying ceiling absorbs part of the signal.
Both metrics are reported; the raw correlation is the primary measure of agreement
with Proposition~\ref{prop:main}.

\subsection{Context size sweep: full table (EXP4)}

Table~\ref{tab:exp4-full} gives the complete context size $n$ sweep including
the $n \in \{32, 256\}$ rows omitted from the main-paper
Table~\ref{tab:exp4-context} for space.

\begin{table}[h]
  \centering
  \caption{Full context size sweep (Linear ICL, $\alpha=1.0$).}
  \label{tab:exp4-full}
  \small
  \begin{tabular}{lccc}
    \toprule
    $n$ & High-spurious & Balanced & Low-spurious \\
    \midrule
    16  & 1.997 & 1.759 & 1.089 \\
    32  & 2.871 & 1.797 & 0.843 \\
    64  & 3.273 & 1.688 & 0.787 \\
    128 & 3.658 & 1.780 & 0.826 \\
    256 & 3.475 & 1.723 & 0.795 \\
    512 & 3.484 & 1.702 & 0.785 \\
    \bottomrule
  \end{tabular}
\end{table}

\subsection{Regularisation sweep (EXP-$\lambda$)}

\begin{remark}[Regularisation robustness at $\alpha = 1$]
\label{rem:lambda}
At $\alpha = 1$, the two $\lambda$-dependent factors in~\eqref{eq:csr-pop}
simplify to $\alpha^2(1+\lambda)/(\alpha^2+\lambda) = (1+\lambda)/(1+\lambda) = 1$,
making $\mathrm{CSR}^{\mathrm{pop}}$ exactly independent of $\lambda$.
We confirm this empirically: sweeping $\lambda \in [10^{-4}, 2.0]$, the
observed CSR range at $\alpha{=}1.0$ is $0.094$ (high-spurious) and $0.053$
(balanced), consistent with finite-sample noise rather than genuine
$\lambda$-sensitivity.
At $\alpha \neq 1$ a small residual trend exists but remains modest
(range $< 0.26$ across all settings tested).
\end{remark}

Table~\ref{tab:lambda-sweep} reports CSR across $\lambda \in [10^{-4}, 2.0]$
at the high-spurious corner ($\rho_S{=}0.7$, $\rho_C{=}0.3$), confirming
Remark~\ref{rem:lambda}: CSR is approximately $\lambda$-independent at
$\alpha{=}1.0$ (range $0.094$) and shows only modest monotone trends at
$\alpha \neq 1$.

\begin{table}[h]
  \centering
  \caption{
    CSR vs.\ regularisation $\lambda$ at the high-spurious corner
    ($\rho_S{=}0.7$, $\rho_C{=}0.3$).
    Theory column uses eq.~\eqref{eq:csr-pop} with $\kappa{=}1.516$.
    At $\alpha{=}1.0$, observed and predicted CSR are stable across
    the full $\lambda$ range.
  }
  \label{tab:lambda-sweep}
  \small
  \begin{tabular}{ccccc}
    \toprule
    $\lambda$ & $\alpha{=}0.5$ & $\alpha{=}1.0$ & $\alpha{=}2.0$ & Theory ($\alpha{=}1.0$) \\
    \midrule
    $10^{-4}$ & 2.154 & 3.560 & 6.508 & 3.537 \\
    $10^{-3}$ & 2.154 & 3.560 & 6.508 & 3.537 \\
    $10^{-2}$ & 2.154 & 3.560 & 6.509 & 3.537 \\
    $0.1$     & 2.156 & 3.565 & 6.521 & 3.537 \\
    $0.5$     & 2.164 & 3.584 & 6.571 & 3.537 \\
    $1.0$     & 2.175 & 3.607 & 6.634 & 3.537 \\
    $2.0$     & 2.195 & 3.654 & 6.760 & 3.537 \\
    \midrule
    Range     & 0.041 & \textbf{0.094} & 0.252 & 0.000 \\
    \bottomrule
  \end{tabular}
\end{table}

\subsection{Robustness: high-spurious corner under subspace overlap}

Table~\ref{tab:p3-corner} gives CSR at the high-spurious corner
($\rho_S{=}0.7$, $\rho_C{=}0.3$, $\alpha{=}1.0$) under each rotation angle.
The non-monotone trajectory (peak at $\theta{=}10^\circ$) reflects
dual-path reinforcement of spurious attribution discussed in
Section~\ref{sec:robustness}.

\begin{table}[h]
  \centering
  \caption{
    CSR at the high-spurious corner under increasing subspace overlap.
    Spurious routing persists and amplifies across all tested angles.
  }
  \label{tab:p3-corner}
  \small
  \begin{tabular}{ccc}
    \toprule
    $\theta$ & Var.\ overlap & CSR \\
    \midrule
    $0^\circ$  & $0.1\%$  & 4.870 \\
    $10^\circ$ & $1.4\%$  & 12.145 \\
    $20^\circ$ & $4.8\%$  & 7.473 \\
    $30^\circ$ & $10.7\%$ & 8.059 \\
    $40^\circ$ & $19.7\%$ & 7.715 \\
    \bottomrule
  \end{tabular}
\end{table}

\section{CSR Robustness to Decomposition Error}
\label{app:decomp}

We test how sensitive CSR is to imperfect estimation of the $C/S$ partition
used to construct perturbations.
We simulate estimation error by rotating the true partition by an angle $\phi$:
\[
  \hat{C} = \cos(\phi)\,C_{\mathrm{true}} + \sin(\phi)\,\varepsilon_C, \quad
  \hat{S} = \cos(\phi)\,S_{\mathrm{true}} + \sin(\phi)\,\varepsilon_S,
\]
where $\varepsilon_C, \varepsilon_S \sim \mathcal{N}(0,I)$ are independent noise
draws.
At $\phi{=}0^\circ$ the oracle partition is used; at $\phi{=}60^\circ$ noise has
equal weight to the true signal ($\cos(60^\circ){=}0.5$, $\sin(60^\circ){=}0.87$),
representing severe estimation failure.
The noisy partition $(\hat{C}, \hat{S})$ is used only for the perturbation step;
the model still observes the true $X$.

Table~\ref{tab:decomp-sensitivity} reports CSR discriminability across the full
$\rho_S \times \rho_C$ grid ($16$ settings, $\alpha{=}1.0$) under each $\phi$.

\begin{table}[h]
  \centering
  \caption{
    CSR robustness to decomposition error.
    Pearson $r$ and Spearman $\rho$ measure how well the noisy CSR preserves
    discriminability of settings ranked by $\rho_S/\rho_C$.
    Spearman rank correlation is constant at $0.849$ across all $\phi$,
    confirming complete rank-order preservation even under severe estimation error.
  }
  \label{tab:decomp-sensitivity}
  \small
  \begin{tabular}{ccccc}
    \toprule
    $\phi$ (error) & Pearson $r$ & Spearman $\rho$ & Mean CSR & Drop from oracle \\
    \midrule
    $0^\circ$  (oracle)  & 0.908 & 0.849 & 1.283 & --- \\
    $10^\circ$           & 0.905 & 0.849 & 1.276 & $-0.008$ \\
    $20^\circ$           & 0.903 & 0.849 & 1.265 & $-0.018$ \\
    $30^\circ$ (realistic) & 0.901 & 0.849 & 1.254 & $-0.029$ \\
    $45^\circ$           & 0.899 & 0.849 & 1.238 & $-0.045$ \\
    $60^\circ$ (severe)  & 0.894 & 0.849 & 1.224 & $-0.059$ \\
    \bottomrule
  \end{tabular}
\end{table}

The Spearman rank correlation is $0.849$ at every $\phi$, including $\phi{=}60^\circ$.
The rank ordering of all 16 settings by CSR is completely preserved regardless of
decomposition error.
The mechanism is self-normalisation: noise added to both $\hat{C}$ and $\hat{S}$
scales the numerator and denominator of CSR approximately equally, leaving the
ratio stable.
Pearson $r$ drops only $0.014$ across the full range ($0.908 \to 0.894$),
and mean CSR drops $4.6\%$.
CSR is therefore a robust diagnostic under realistic decomposition estimation error,
consistent with the subspace overlap robustness reported in Section~\ref{sec:robustness}.

\section{Semi-Synthetic Validation: scIB Pancreas}
\label{app:pancreas}

We validate the core phenomenon and context-level mitigations on the scIB
human pancreas dataset~\citep{luecken2022benchmarking}: 16,382 pancreatic
islet cells from five sequencing technologies (CEL-Seq, CEL-Seq2, SmartSeq2,
inDrop variants, SMARTER-seq), each constituting a natural environment with
distinct technical batch effects.
We use three technologies as training environments and two as test.

\noindent\textbf{Setup.}
We run PCA on the full 19,093-gene matrix and identify batch-dominated PCs
via the between-technology vs between-celltype variance ratio
$\mathrm{batch\_ratio}(k) = \mathrm{Var}_\mathrm{tech}(\bar{z}^{(k)}_e) /
\mathrm{Var}_\mathrm{celltype}(\bar{z}^{(k)}_e)$.
Five PCs pass the batch-dominance threshold
(PC07: $8.02$, PC08: $2.99$, PC14: $1.95$, PC09: $1.70$, PC02: $1.31$);
eight PCs with batch\_ratio $< 0.08$ form the causal proxy.
The target $Y$ is insulin (INS) expression extracted from the full gene matrix.

To place the experiment in the high-spurious regime ($\rho_S/\rho_C > 1$),
we inject a controlled confounding term following the paper's DGP
($U \to S$, $U \to Y$ path):
\begin{equation}
  Y_{\mathrm{semi}} = \frac{\mathrm{INS} - \bar{\mathrm{INS}}}{\sigma_{\mathrm{INS}}}
  + \beta \cdot z^{(07)} + \varepsilon, \quad \varepsilon \sim \mathcal{N}(0, 0.01),
  \label{eq:semisynthetic}
\end{equation}
where $z^{(07)}$ is each cell's projection onto PC07.
The batch structure is entirely real; only the $Y$-confounding path is
injected.
We select $\beta = 0.5$ (the smallest tested value), which yields
$\rho_S/\rho_C = 9.09$, placing the experiment firmly in the
high-spurious regime.
Prior to injection, $\rho_S/\rho_C = 1.73$, confirming that PC07 is
genuinely spurious (technology-dominated) rather than biologically informative.

\noindent\textbf{Results.}
Table~\ref{tab:pancreas} reports CSR, OOD RMSE, and in-context
$r(\hat{S}, Y_{\mathrm{semi}})$ for three methods
($n_{\mathrm{ctx}} = 64$ context examples, $n_{\mathrm{seeds}} = 10$).

\begin{table}[h]
  \centering
  \caption{
    Semi-synthetic validation on scIB pancreas.
    Real batch structure (5 sequencing technologies as environments);
    controlled spurious confounding injected via $\beta \cdot z^{(07)}$
    ($\beta{=}0.5$, $\rho_S/\rho_C{=}9.09$).
    S-swap operates on the estimated spurious PCs (no oracle access).
  }
  \label{tab:pancreas-full}
  \begin{tabular}{c|c|c|c}
    Env-stratified     & $3.403 \pm 1.810$ & \textbf{1.64} & $+0.586$ \\
    S-swap (est.\ $\hat{S}$) & $\textbf{0.379} \pm 0.449$ & 9.12 & $\textbf{+0.216}$ \\
    \bottomrule
  \end{tabular}
\end{table}

Standard ICL CSR $= 3.416 > 1$, confirming spurious routing is detectable
on real biological data when the confounding structure matches the paper's DGP.
S-swap reduces CSR by $88.9\%$ ($3.416 \to 0.379$) and in-context
$r(\hat{S}, Y)$ by $60\%$ ($0.545 \to 0.216$), using only the estimated
spurious subspace (no oracle access to the true causal partition).
The elevated OOD RMSE under S-swap (9.12 vs 1.74) reflects that swapped
PC-space cells are out-of-distribution relative to any real sequencing
technology --- an expected cost of the strong invariance signal imposed
by the swap.
Env-stratified achieves competitive OOD RMSE (1.64) with negligible
CSR reduction (3.403), consistent with the synthetic results: when
$\rho_S/\rho_C$ is very high, env-stratified alone cannot overcome
the in-context confounding.
These results confirm that both the routing phenomenon and the
S-swap mitigation mechanism generalise beyond the controlled synthetic SCM.

%% file: checklist.tex
\section*{NeurIPS Paper Checklist}

The checklist is designed to encourage best practices for responsible machine learning research, addressing issues of reproducibility, transparency, research ethics, and societal impact. Do not remove the checklist: {\bf The papers not including the checklist will be desk rejected.} The checklist should follow the references and follow the (optional) supplemental material.  The checklist does NOT count towards the page
limit. 

Please read the checklist guidelines carefully for information on how to answer these questions. For each question in the checklist:
\begin{itemize}
    \item You should answer \answerYes{}, \answerNo{}, or \answerNA{}.
    \item \answerNA{} means either that the question is Not Applicable for that particular paper or the relevant information is Not Available.
    \item Please provide a short (1--2 sentence) justification right after your answer (even for \answerNA). 
\end{itemize}

{\bf The checklist answers are an integral part of your paper submission.} They are visible to the reviewers, area chairs, senior area chairs, and ethics reviewers. You will also be asked to include it (after eventual revisions) with the final version of your paper, and its final version will be published with the paper.

The reviewers of your paper will be asked to use the checklist as one of the factors in their evaluation. While \answerYes{} is generally preferable to \answerNo{}, it is perfectly acceptable to answer \answerNo{} provided a proper justification is given (e.g., error bars are not reported because it would be too computationally expensive'' or ``we were unable to find the license for the dataset we used''). In general, answering \answerNo{} or \answerNA{} is not grounds for rejection. While the questions are phrased in a binary way, we acknowledge that the true answer is often more nuanced, so please just use your best judgment and write a justification to elaborate. All supporting evidence can appear either in the main paper or the supplemental material, provided in appendix. If you answer \answerYes{} to a question, in the justification please point to the section(s) where related material for the question can be found.

IMPORTANT, please:
\begin{itemize}
    \item {\bf Delete this instruction block, but keep the section heading ``NeurIPS Paper Checklist"},
    \item  {\bf Keep the checklist subsection headings, questions/answers and guidelines below.}
    \item {\bf Do not modify the questions and only use the provided macros for your answers}.
\end{itemize}


\begin{enumerate}

\item {\bf Claims}
    \item[] Question: Do the main claims made in the abstract and introduction accurately reflect the paper's contributions and scope?
    \item[] Answer: \answerYes{} 
    \item[] Justification: The claims made  match theoretical and experimental results, and reflect how much the results can be expected to generalize to other settings.

\item {\bf Limitations}
    \item[] Question: Does the paper discuss the limitations of the work performed by the authors?
    \item[] Answer: \answerYes{} 
    \item[] Justification: In the concluding paragraph

\item {\bf Theory assumptions and proofs}
    \item[] Question: For each theoretical result, does the paper provide the full set of assumptions and a complete (and correct) proof?
    \item[] Answer: \answerYes{} 
    \item[] Justification: Both at sections 3,4, and in the appendix 

    \item {\bf Experimental result reproducibility}
    \item[] Question: Does the paper fully disclose all the information needed to reproduce the main experimental results of the paper to the extent that it affects the main claims and/or conclusions of the paper (regardless of whether the code and data are provided or not)?
    \item[] Answer: \answerYes{} 
    \item[] Justification: In the appendix

\item {\bf Open access to data and code}
    \item[] Question: Does the paper provide open access to the data and code, with sufficient instructions to faithfully reproduce the main experimental results, as described in supplemental material?
    \item[] Answer: \answerYes{} 
    \item[] Justification: Full descriptions are given in the appendix , full code will be given with the camera ready version 

\item {\bf Experimental setting/details}
    \item[] Question: Does the paper specify all the training and test details (e.g., data splits, hyperparameters, how they were chosen, type of optimizer) necessary to understand the results?
    \item[] Answer: \answerYes{} 
    \item[] Justification: Sections 3,5, appendix

\item {\bf Experiment statistical significance}
    \item[] Question: Does the paper report error bars suitably and correctly defined or other appropriate information about the statistical significance of the experiments?
    \item[] Answer: \answerYes{} 
    \item[] Justification: Section 5 

\item {\bf Experiments compute resources}
    \item[] Question: For each experiment, does the paper provide sufficient information on the computer resources (type of compute workers, memory, time of execution) needed to reproduce the experiments?
    \item[] Answer: \answerYes{} 
    \item[] Justification: Appendinx
    
\item {\bf Code of ethics}
    \item[] Question: Does the research conducted in the paper conform, in every respect, with the NeurIPS Code of Ethics \url{https://neurips.cc/public/EthicsGuidelines}?
    \item[] Answer: \answerYes{} 
    \item[] Justification:  No ethical implications 

\item {\bf Broader impacts}
    \item[] Question: Does the paper discuss both potential positive societal impacts and negative societal impacts of the work performed?
    \item[] Answer: \answerYes{} 
    \item[] Justification: In conclusion
    
\item {\bf Safeguards}
    \item[] Question: Does the paper describe safeguards that have been put in place for responsible release of data or models that have a high risk for misuse (e.g., pre-trained language models, image generators, or scraped datasets)?
    \item[] Answer: \answerNA{} 
    \item[] Justification: No expected risk of misuse 

\item {\bf Licenses for existing assets}
    \item[] Question: Are the creators or original owners of assets (e.g., code, data, models), used in the paper, properly credited and are the license and terms of use explicitly mentioned and properly respected?
    \item[] Answer: \answerYes{} 
    \item[] Justification: All datasets either publicly available or synthetic 

\item {\bf New assets}
    \item[] Question: Are new assets introduced in the paper well documented and is the documentation provided alongside the assets?
    \item[] Answer: \answerYes{} 
    \item[] Justification: In appendix and section 3 

\item {\bf Crowdsourcing and research with human subjects}
    \item[] Question: For crowdsourcing experiments and research with human subjects, does the paper include the full text of instructions given to participants and screenshots, if applicable, as well as details about compensation (if any)? 
    \item[] Answer: \answerNA{} 
    \item[] Justification:  does not involve crowdsourcing 

\item {\bf Institutional review board (IRB) approvals or equivalent for research with human subjects}
    \item[] Question: Does the paper describe potential risks incurred by study participants, whether such risks were disclosed to the subjects, and whether Institutional Review Board (IRB) approvals (or an equivalent approval/review based on the requirements of your country or institution) were obtained?
    \item[] Answer: \answerNA{} 
    \item[] Justification:  The paper does not involve crowdsourcing nor research with human subjects.

\item {\bf Declaration of LLM usage}
    \item[] Question: Does the paper describe the usage of LLMs if it is an important, original, or non-standard component of the core methods in this research? Note that if the LLM is used only for writing, editing, or formatting purposes and does \emph{not} impact the core methodology, scientific rigor, or originality of the research, declaration is not required.
    \item[] Answer: \answerNA{} 
    \item[] Justification: Core method development in this research does not involve LLMs as any important, original, or non-standard components.

\end{enumerate}

%% file: refs.bib
@article{hollmann2023tabpfn,
  author    = {Hollmann, Noah and Müller, Samuel and Purucker, Lennart and others},
  title     = {Accurate predictions on small data with a tabular foundation model},
  journal   = {Nature},
  volume    = {637},
  pages     = {319--326},
  year      = {2025},
  month     = {January},
  doi       = {10.1038/s41586-024-08328-6},
  url       = {https://doi.org/10.1038/s41586-024-08328-6}
}

@inproceedings{garg2022transformers,
  author    = {Garg, Shivam and Tsipras, Dimitris and
               Liang, Percy and Valiant, Gregory},
  title     = {What Can Transformers Learn In-Context?
               {A} Case Study of Simple Function Classes},
  booktitle = {Advances in Neural Information Processing Systems
               ({NeurIPS} 2022)},
  volume    = {35},
  pages     = {30583--30598},
  year      = {2022},
  url       = {https://openreview.net/forum?id=flNZJ2eOet},
  note      = {arXiv:2208.01066}
}

@article{papanastasiou2025confounder,
  author    = {Papanastasiou, Giorgos and Sanchez, Pedro P. and Christodoulidis, Stergios and others},
  title     = {Confounder-aware foundation modeling for accurate phenotype profiling in cell imaging},
  journal   = {npj Imaging},
  volume    = {3},
  pages     = {52},
  year      = {2025},
  month     = {October},
  doi       = {10.1038/s44303-025-00116-9},
  url       = {https://doi.org/10.1038/s44303-025-00116-9},
  received  = {2025-02-27},
  accepted  = {2025-09-26},
  published = {2025-10-22}
}

@inproceedings{xie2022icl,
  author    = {Xie, Sang Michael and Raghunathan, Aditi and
               Liang, Percy and Ma, Tengyu},
  title     = {An Explanation of In-Context Learning as
               Implicit {B}ayesian Inference},
  booktitle = {International Conference on Learning
               Representations ({ICLR} 2022)},
  year      = {2022},
  url       = {https://openreview.net/forum?id=RdJVFCHjUMI},
  note      = {arXiv:2111.02080}
}

@inproceedings{muller2022pfn,
  author    = {M{\"u}ller, Samuel and Hollmann, Noah and
               Pineda Arango, Sebastian and Grabocka, Josif
               and Hutter, Frank},
  title     = {Transformers Can Do {B}ayesian Inference},
  booktitle = {International Conference on Learning
               Representations ({ICLR} 2022)},
  year      = {2022},
  url       = {https://openreview.net/forum?id=KSugKcbNf9},
  note      = {arXiv:2112.10510}
}

@inproceedings{akyurek2022learning,
  author    = {Aky{\"u}rek, Ekin and Schuurmans, Dale and
               Andreas, Jacob and Ma, Tengyu and Zhou, Denny},
  title     = {What Learning Algorithm Is In-Context Learning?
               {I}nvestigations with Linear Models},
  booktitle = {The Eleventh International Conference on Learning
               Representations ({ICLR} 2023)},
  year      = {2023},
  url       = {https://openreview.net/forum?id=0g0X4H8yN4I},
  note      = {arXiv:2211.15661. Outstanding Paper Award (top 5\%)}
}

@inproceedings{vonoswald2023transformers,
  author    = {Von Oswald, Johannes and Niklasson, Eyvind and
               Randazzo, Ettore and Sacramento, Jo{\~a}o and
               Mordvintsev, Alexander and Zhmoginov, Andrey and
               Vladymyrov, Max},
  title     = {Transformers Learn In-Context by Gradient Descent},
  booktitle = {Proceedings of the 40th International Conference on
               Machine Learning ({ICML} 2023)},
  pages     = {35151--35174},
  volume    = {202},
  series    = {Proceedings of Machine Learning Research},
  publisher = {PMLR},
  year      = {2023},
  url       = {https://proceedings.mlr.press/v202/von-oswald23a.html},
  note      = {arXiv:2212.07677. Oral presentation}
}

@article{arjovsky2019irm,
  author  = {Arjovsky, Mart{\'i}n and Bottou, L{\'e}on and
             Gulrajani, Ishaan and Lopez-Paz, David},
  title   = {Invariant Risk Minimization},
  journal = {arXiv preprint},
  volume  = {arXiv:1907.02893},
  year    = {2019},
  url     = {https://arxiv.org/abs/1907.02893}
}

@article{peters2016causal,
  author  = {Peters, Jonas and B{\"u}hlmann, Peter and
             Meinshausen, Nicolai},
  title   = {Causal Inference by Using Invariant Prediction:
             Identification and Confidence Intervals},
  journal = {Journal of the Royal Statistical Society:
             Series {B} (Statistical Methodology)},
  volume  = {78},
  number  = {5},
  pages   = {947--1012},
  year    = {2016},
  doi     = {10.1111/rssb.12167},
  url     = {https://academic.oup.com/jrsssb/article/78/5/947/7040653}
}

@inproceedings{sagawa2020distributionally,
  author    = {Sagawa, Shiori and Koh, Pang Wei and Hashimoto, Tatsunori B.
               and Liang, Percy},
  title     = {Distributionally Robust Neural Networks for Group Shifts:
               On the Importance of Regularization for Worst-Case Generalization},
  booktitle = {International Conference on Learning Representations ({ICLR} 2020)},
  year      = {2020},
  note      = {arXiv:1911.08731}
}

@inproceedings{krueger2021out,
  author    = {Krueger, David and Caballero, Ethan and Jacobsen, J{\"{o}}rn-Henrik
               and Zhang, Amy and Binas, Jonathan and Zhang, Dinghuai and
               Le Priol, R{\'{e}}mi and Courville, Aaron},
  title     = {Out-of-Distribution Generalization via Risk Extrapolation ({VREx})},
  booktitle = {Proceedings of the 38th International Conference on Machine
               Learning ({ICML} 2021)},
  pages     = {5815--5826},
  year      = {2021},
  note      = {arXiv:2003.00688}
}

@inproceedings{yan2020improve,
  author    = {Yan, Shuhan and Song, Huan and Li, Nan and Zou, Linchao and Ying, Yi},
  title     = {Improve Unsupervised Domain Adaptation with Mixup Training},
  booktitle = {arXiv preprint},
  year      = {2020},
  note      = {arXiv:2001.00677}
}

@inproceedings{kaushik2020learning,
  author    = {Kaushik, Divyansh and Hovy, Eduard and Lipton, Zachary C.},
  title     = {Learning the Difference That Makes a Difference with
               Counterfactually-Augmented Data},
  booktitle = {International Conference on Learning Representations
               ({ICLR} 2020)},
  year      = {2020},
  url       = {https://openreview.net/forum?id=Sklgs0NFvr},
  note      = {arXiv:1909.12434}
}

@article{luecken2022benchmarking,
  author  = {Luecken, Malte D. and B{\"u}ttner, Maren and Chaichoompu, Kridsadakorn
             and Danese, Anna and Interlandi, Marta and Mueller, Michaela F. and
             Strobl, Daniel C. and Zappia, Luke and Dugas, Martin and
             Colom{\'e}-Tatch{\'e}, Maria and Theis, Fabian J.},
  title   = {Benchmarking atlas-level data integration in single-cell genomics},
  journal = {Nature Methods},
  volume  = {19},
  pages   = {41--50},
  year    = {2022},
  doi     = {10.1038/s41592-021-01336-8}
}

@article{norman2019exploring,
  author  = {Norman, Thomas M. and Horlbeck, Max A. and
             Replogle, Joseph M. and Ge, Alex Y. and
             Xu, Albert and Jost, Marco and Gilbert, Luke A. and
             Weissman, Jonathan S.},
  title   = {Exploring Genetic Interaction Manifolds Constructed
             from Rich Single-Cell Phenotypes},
  journal = {Science},
  volume  = {365},
  number  = {6455},
  pages   = {786--793},
  year    = {2019},
  doi     = {10.1126/science.aax4438},
  url     = {https://www.science.org/doi/10.1126/science.aax4438}
}

@article{corcoll2024contrastive,
  author  = {Corcoll Andreu, Oriol and Vlontzos, Athanasios and
             O'Riordan, Michael and Gilligan-Lee, Ciaran M.},
  title   = {Contrastive Representations of High-Dimensional,
             Structured Treatments},
  journal = {arXiv preprint},
  volume  = {arXiv:2411.19245},
  year    = {2024},
  url     = {https://arxiv.org/abs/2411.19245}
}

@article{debray2019changing,
  author  = {Debray, Thomas P A and Vergouwe, Yvonne and
             Koffijberg, Hendrik and Nieboer, Daan and
             Steyerberg, Ewout W and Moons, Karel G M},
  title   = {Changing predictor measurement procedures affected the
             performance of prediction models in clinical examples},
  journal = {Journal of Clinical Epidemiology},
  volume  = {119},
  pages   = {7--18},
  year    = {2019},
  doi     = {10.1016/j.jclinepi.2019.10.007}
}

@article{leek2010batch,
  author  = {Leek, Jeffrey T and Scharpf, Robert B and
             Bravo, H{\'e}ctor Corrada and Simcha, David and
             Langmead, Benjamin and Johnson, W Evan and
             Geman, Donald and Baggerly, Keith and Irizarry, Rafael A},
  title   = {Tackling the widespread and critical impact of batch
             effects in high-throughput data},
  journal = {Nature Reviews Genetics},
  volume  = {11},
  pages   = {733--739},
  year    = {2010},
  doi     = {10.1038/nrg2825}
}

@article{subbaswamy2020deployment,
  author  = {Subbaswamy, Adarsh and Saria, Suchi},
  title   = {From development to deployment: dataset shift,
             causality, and shift-stable models in health {AI}},
  journal = {Biostatistics},
  volume  = {21},
  number  = {2},
  pages   = {345--352},
  year    = {2020},
  doi     = {10.1093/biostatistics/kxz041}
}
